\documentclass{article}
\usepackage{iclr2027_conference,times}

\usepackage{amsmath,amsfonts,bm}

\def\eqref#1{equation~\ref{#1}}

\def\1{\bm{1}}

\DeclareMathAlphabet{\mathsfit}{\encodingdefault}{\sfdefault}{m}{sl}
\SetMathAlphabet{\mathsfit}{bold}{\encodingdefault}{\sfdefault}{bx}{n}

\usepackage{hyperref}
\usepackage{url}
\usepackage{booktabs}
\usepackage{multirow}
\usepackage{amsthm}
\usepackage{xcolor}
\usepackage{pifont}
\usepackage[nointegrals]{wasysym}
\usepackage{tikz}
\usepackage{float}
\usepackage{placeins}
\usetikzlibrary{arrows.meta,fit}

\newtheorem{proposition}{Proposition}

\title{Tsubame: Tree Replay for \\Diffusion-Based Speculative Decoding}

\author{Yepeng Weng$^{1}$, Qiao Hu$^{2}$\thanks{Corresponding Author.} , Takehisa Yairi$^{1}$ \\
$^{1}$The University of Tokyo \\
$^{2}$National Center for Mathematics and Interdisciplinary Sciences (NCMIS), AMSS, CAS \\
\texttt{yweng@g.ecc.u-tokyo.ac.jp, huqiao2020@amss.ac.cn} \\
}

\iclrfinalcopy

\begin{document}

\maketitle
\lhead{Preprint.}

\begin{abstract}
Context-aware dynamic trees allocate the speculative decoding budget according to draft path probabilities, adapting their depth and branching to the current context. Under stochastic decoding, however, we find that this structural advantage does not always compensate for the acceptance gains of random sampling paired with advanced verification, and such dynamic trees can fall behind sampled chains in some settings. These trees grow their topology from the candidates themselves, so the tokens submitted for verification are typically the deterministic high-score tokens selected during construction. This coupling is not inherent: once the topology is fixed, its nodes can be repopulated by sampling, allowing dynamic trees to retain their structural advantage while also benefiting from random sampling and advanced verification. Diffusion-based drafters make this practical, as their parallel outputs or lightweight conditional corrections allow candidates to be regenerated cheaply after the complete topology is known. We introduce \textbf{Tsubame}, a two-pass tree speculative decoding framework for diffusion-based drafters. The first pass plans and freezes a context-aware topology using draft path scores; the second replays the fixed topology, sampling the tokens that populate its nodes to form the candidate tree for verification. We prove that Tsubame is lossless under compatible sampling and verification strategies. Experiments across three diffusion-based drafters, six datasets, and multiple candidate budgets show that Tsubame improves acceptance length and throughput over deterministic trees, including settings where it reverses their disadvantage against sampled chains.
\end{abstract}

\section{Introduction}
\label{sec:introduction}

Autoregressive language models generate text one token at a time, requiring a separate forward pass for each new token. Speculative decoding reduces this sequential bottleneck by drafting multiple tokens with a lightweight model and verifying them in parallel while preserving the target model's output distribution \citep{leviathan2023fast,chen2023speculative}. Tree-based methods extend a single draft chain into multiple candidate continuations. In particular, context-aware dynamic trees deterministically select the tokens with the highest draft path scores to fill a candidate budget, so that the tree's depth and branching adapt to the current context \citep{li2024eagle2,ringel2026ddtree,lin2026dominotree}.

However, a stronger tree structure does not necessarily yield a stronger speculative decoder. Under stochastic decoding, the way candidates are generated and verified also matters. A sampled chain can be verified with Rejection Sampling or with Block Verification; the latter jointly verifies a draft block and is optimal in the expected number of tokens generated per iteration in single-chain speculative decoding \citep{sun2025block}. In our experiments, a deterministic dynamic tree outperforms the top-1 chain baseline, but falls behind a sampled chain with Block Verification on some tasks. Thus, comparing only against a deterministic chain can obscure an important limitation of dynamic trees.

This limitation stems from how dynamic trees select their candidates. Draft path scores guide the allocation of nodes to promising continuations, but local acceptance is then limited by the target probability mass of the retained high-score candidates. Random sampling, paired with appropriate verification, offers a different source of acceptance gains that deterministic selection does not exploit. Recent work has begun to introduce stochastic candidates within the standard dynamic-tree construction process \citep{rheosampling2026}. We ask a more radical question: \emph{can the planning of a tree's topology be fully separated from the sampling of its candidates?}

We answer this question with \textbf{Tsubame}, a two-pass tree speculative decoding framework that separates whole-tree topology planning from final candidate generation. The first pass uses draft path scores to construct and freeze a context-aware topology. The second pass then \emph{replays} this frozen topology: it keeps the same structure and samples fresh tokens for every node, and these tokens, rather than the deterministic planning candidates, are submitted for target-model verification. We refer to this procedure as \emph{tree replay}. This design retains the first pass's context-aware structural allocation while enabling random sampling and its associated verification rules. With autoregressive drafters, replay would require rerunning a draft forward pass at every tree layer under the replayed tokens, effectively repeating the entire drafting stage. Diffusion-based drafters make this cost small: future-position distributions are either available from parallel prediction or can be obtained under new prefixes through lightweight conditional computation \citep{chen2026dflash,huang2026domino,cheng2026dspark,inco2026dflash2}, making full-tree replay practical.

We evaluate Tsubame with three diffusion-based drafters, Domino, DFlash2, and DSpark, on six datasets spanning math, code, and conversation. Paired with different sampling strategies and verifiers, replay can improve over the deterministic tree, with tree-level verifiers realizing these gains most consistently. In particular, with mixed sampling and UniVer, replay improves acceptance length and throughput over both the deterministic first-pass tree and sampled chains with Block Verification on all six datasets for Domino and DFlash2, reversing every task-level deficit of the deterministic tree against Block Verification. These gains persist as the candidate budget grows from 16 to 64, and in concurrent serving, replay with tree-level verification outperforms both the deterministic tree and Block Verification at every evaluated concurrency. Together, these results show that a well-planned topology and effective sampling and verification offer complementary, rather than competing, routes to higher acceptance and throughput.

Our contributions are threefold:
\begin{itemize}
    \item We point out a largely overlooked gap in how dynamic trees are evaluated under stochastic decoding: comparisons against top-1 chains overstate the value of tree structure, as budget-matched sampled chains with advanced verification can outperform deterministic dynamic trees and thus serve as a more informative baseline.
    \item We propose Tsubame, a two-pass framework that freezes a context-aware topology in the first pass and replays it in the second, resampling every node for verification. We prove that Tsubame is lossless under any compatible sampling--verification pair.
    \item We show that Tsubame generalizes across diffusion-based drafters, tasks, and target models. Paired with suitable sampling and verification, it consistently improves acceptance length and throughput over deterministic trees, reverses their deficits against sampled chains, and retains its gains across candidate budgets and under concurrent serving.
\end{itemize}

\section{Related Work and Background}
\label{sec:background}

\subsection{Autoregressive and Diffusion-Based Drafters}
\label{sec:drafters}

Speculative decoding admits different forms of draft models. The original methods use a separate, smaller autoregressive language model to propose continuations for a larger target model \citep{leviathan2023fast,chen2023speculative}. The EAGLE series instead trains a lightweight autoregressive draft model that takes target-model features as input \citep{li2024eagle,li2025eagle3}. These methods allow the target model to verify multiple positions at once, but the drafter itself remains autoregressive: each draft forward pass advances the candidates by only one position.

Diffusion-based drafters address this remaining sequential bottleneck through parallel block prediction. DFlash produces distributions for multiple future positions in a single forward pass \citep{chen2026dflash}. Domino, DSpark, and DFlash2 further model within-block dependencies through lightweight mechanisms that avoid a full backbone pass for each draft position: Domino corrects logits with a GRU state advanced along the prefix, DSpark applies a Markov-head correction conditioned on the preceding token, and DFlash2 uses precomputed transition scores between candidates \citep{huang2026domino,cheng2026dspark,inco2026dflash2}.

\subsection{Tree-Structured Speculative Decoding}
\label{sec:tree-decoding}

Tree speculative decoding proposes multiple continuation paths \citep{miao2024specinfer,jeon2024rsd,xiong2024dyspec}. Static methods use predefined or offline-optimized topologies, as in EAGLE and Sequoia \citep{li2024eagle,chen2024sequoia}. Context-aware methods such as EAGLE-2 and OPT-Tree instead use draft probabilities as a surrogate for acceptance, deterministically selecting candidates with high cumulative path probabilities to fill the verification budget \citep{li2024eagle2,wang2025opttree}, allowing the tree to adapt its depth and branching to the context.

Different drafters can also lead to different tree-construction strategies. Autoregressive methods typically expand trees layer by layer, as in EAGLE-2's expand-then-prune procedure \citep{li2024eagle2}. Diffusion-based drafters make node-wise expansion inexpensive: DDTree uses max-heap search over parallel draft distributions produced by DFlash, while DominoTree uses Domino as its drafter and performs the corresponding lightweight conditional updates during expansion without rerunning the full diffusion backbone \citep{ringel2026ddtree,lin2026dominotree}. Despite their different construction orders, these dynamic methods use draft path scores to select candidates deterministically and retain the selected tokens for verification.

\subsection{Candidate Sampling and Verification}
\label{sec:sampling-verification}

Under stochastic decoding ($T>0$), acceptance depends on both candidate generation and verification. At a fixed prefix, let $p$ and $q$ denote the target and draft next-token distributions over vocabulary $\mathcal{V}$. Standard speculative decoding samples a candidate from $q$ and applies rejection sampling (RS), accepting token $v$ with probability $\min\{1,p(v)/q(v)\}$ and correcting the output distribution upon rejection \citep{leviathan2023fast,chen2023speculative}. Its marginal acceptance probability is
\begin{equation}
    \alpha_{\mathrm{RS}}(p,q)
    = \sum_{v\in\mathcal{V}} \min\{p(v),q(v)\}.
    \label{eq:rs-overlap}
\end{equation}
Alternatively, deterministically proposing $v^\star=\arg\max_v q(v)$ and matching it against a target sample gives acceptance probability $p(v^\star)$. Both approaches preserve the target distribution, but their acceptance probabilities differ. \citet{li2026breaking} report that RS often achieves higher acceptance rates than top-1 target-sample matching in native MTP models. Thus, even for a single draft token, how the candidate is generated and verified affects acceptance independently of any tree structure. With multiple candidates and positions, this design space expands to sampling schemes and multi-step verification rules, which we review in Appendix~\ref{app:sampling-verification}.

\subsection{Context-Aware Topology and Stochastic Proposals}
\label{sec:topology-sampling}

Static tree designs, such as EAGLE-1, SpecInfer, and Sequoia, specify the topology in advance \citep{li2024eagle,miao2024specinfer,chen2024sequoia}. In contrast, the context-aware constructions in EAGLE-2, OPT-Tree, DDTree, and DominoTree use candidate tokens and their draft path scores to adapt the topology \citep{li2024eagle2,wang2025opttree,ringel2026ddtree,lin2026dominotree}. As Table~\ref{tab:topology-sampling} shows, static trees leave candidate sampling and verification as a free choice, whereas dynamic trees largely do not. This is not a coincidence: introducing sampled candidates into a dynamic tree faces a substantive difficulty.

\begin{table}[t]
    \centering
    \caption{Comparison of topology construction and sampling/verifier flexibility. Tsubame separates context-aware topology planning from the choice of candidate sampling and verification.}
    \label{tab:topology-sampling}
    \footnotesize
    \setlength{\tabcolsep}{3pt}
    \renewcommand{\arraystretch}{1.2}
    \newcommand{\tableyes}{\textcolor{green!45!black}{\ding{51}}}
    \newcommand{\tableno}{\textcolor{red!75!black}{\ding{55}}}
    \newcommand{\tablepartial}{\textcolor[HTML]{C9A000}{\LEFTcircle}}
    \begin{tabular*}{\textwidth}{@{\extracolsep{\fill}}llllcc@{}}
        \toprule
        \textbf{Tree} & \textbf{Method} & \textbf{Drafter} & \textbf{Construction} & \textbf{Context-aware} & \textbf{Flexible Samp. \& Verif.}$^{*}$ \\
        \midrule
        \multirow{3}{*}{Static} & EAGLE-1 & AR & Handcrafted & \tableno & \tableyes \\
        & SpecInfer & AR & Fixed branching & \tableno & \tableyes \\
        & Sequoia & AR & Offline DP & \tableno & \tableyes \\
        \midrule
        \multirow{5}{*}{Dynamic} & EAGLE-2/3 & AR & Expand-then-prune & \tableyes & \tableno \\
        & RheoSampling & AR & EAGLE-2-style & \tableyes & \tablepartial \\
        & DDTree & Diffusion & Max-heap & \tableyes & \tableno \\
        & DominoTree & Diffusion+GRU & Max-heap & \tableyes & \tableno \\
        \cmidrule(l){2-6}
        & \textbf{Tsubame} & Diffusion-based & Plan + Replay & \tableyes & \tableyes \\
        \bottomrule
    \end{tabular*}
    \par\smallskip
    \begin{minipage}{\textwidth}
        \footnotesize
        $^{*}$Flexibility denotes the ability to choose candidate sampling and proposal-aware verification separately from topology planning. For static trees, the mark denotes support in principle rather than implementation or evaluation in the original work. \tablepartial\ denotes RheoSampling's restricted support: one sampled candidate per local deterministic candidate group, placed according to prescribed rules.
    \end{minipage}
\end{table}

In a static tree, the topology is fixed before any token is drafted. Sampled tokens only fill predetermined slots, and their realizations affect neither the structure nor the survival of any predefined node, including their own. The candidates submitted for verification therefore follow the sampling distribution without distortion, and any compatible sampling--verification pair remains lossless. In a dynamic tree, the opposite holds: the topology is determined by the candidates themselves, as whether a node is expanded, retained, or pruned depends on its own draft path score. If a sampled token directly enters such a construction, its survival depends on its realized token, so the retained candidates no longer follow the sampling distribution, and a verifier that assumes this distribution loses its losslessness guarantee.

RheoSampling addresses this issue within dynamic-tree construction by assigning each sampled candidate a prescribed proxy probability for structural scoring, separate from the actual sampling probability used for verification \citep{rheosampling2026}. This keeps the structural role of sampled candidates independent of their realization, but restricts sampling to one candidate per local deterministic group, placed according to prescribed rules.

Tsubame takes a more direct route: if a sampled token's realization must not influence which nodes survive, the structure can simply be fixed before any token is sampled. It plans the topology first and fills it afterwards, reducing candidate sampling and verification to the static-tree setting. As discussed in Section~\ref{sec:introduction}, diffusion-based drafters make this second pass practical.

\section{Tsubame: Tree Planning and Replay}
\label{sec:method}

Tsubame first plans a context-aware tree topology, then replays the fixed topology to generate the candidates used for verification (Figure~\ref{fig:tsubame-overview}). This separation allows the planner and the sampling--verification scheme to be chosen independently.

\FloatBarrier
\begin{figure}[t]
\centering
\begingroup
\definecolor{planink}{HTML}{426592}
\definecolor{replayink}{HTML}{238477}
\definecolor{mutedink}{HTML}{697586}
\definecolor{ruleink}{HTML}{B6C0CD}
\definecolor{titleink}{HTML}{253449}
\resizebox{\textwidth}{!}{%
\begin{tikzpicture}[x=1cm,y=0.85cm,font=\sffamily\small,
    tokencircle/.style={circle,draw=planink!60,fill=planink!7,minimum size=5.8mm,inner sep=0pt,text=planink},
    roottoken/.style={circle,draw=ruleink,line width=0.9pt,fill=white,text=titleink,minimum size=5.8mm,inner sep=0pt},
    treeedge/.style={draw=ruleink,line width=1pt},
    headline/.style={font=\sffamily\bfseries\large,text=titleink,anchor=west},
    annotation/.style={font=\sffamily\footnotesize,text=mutedink,align=center},
    headingframe/.style={draw=ruleink,line width=0.7pt,rounded corners=3pt,inner sep=3pt},
    transfer/.style={-{Stealth[length=2mm]},draw=mutedink,line width=0.8pt}]
\path[use as bounding box] (-0.1,0.6) rectangle (15,7);
\node[headline,text=planink] (planHeading) at (0.2,6.5) {First Pass};
\node[headline,text=mutedink] (slotsHeading) at (5.4,6.5) {Retain Topology};
\node[headline,text=replayink] (replayHeading) at (10.6,6.5) {Second Pass};
\node[annotation,anchor=west,text=planink!85] (planSubtitle) at (0.2,5.95) {Context-aware construction};
\node[annotation,anchor=west,text=mutedink!85] (slotsSubtitle) at (5.4,5.95) {Discard tokens, keep slots};
\node[annotation,anchor=west,text=replayink!85] (replaySubtitle) at (10.6,5.95) {Tree replay};
\foreach \offset/\tree in {0/plan,5.2/slots,10.4/replay} {
    \begin{scope}[shift={(\offset,0)},xshift=3mm]
        \node[roottoken] (\tree Root) at (2.25,4.95) {$r$};
        \coordinate (\tree Apos) at (1.15,3.9);
        \coordinate (\tree Bpos) at (3.3,3.9);
        \coordinate (\tree Cpos) at (0.6,2.85);
        \coordinate (\tree Dpos) at (1.75,2.85);
        \coordinate (\tree Epos) at (3.3,2.85);
        \coordinate (\tree Fpos) at (0.6,1.8);
    \end{scope}
}
\foreach \slot/\token in {A/a_1,B/a_2,C/b_1,D/b_2,E/b_3,F/c_1}
    \node[tokencircle] (plan\slot) at (plan\slot pos) {$\token$};
\foreach \slot in {A,B,C,D,E,F}
    \node[tokencircle,draw=mutedink!65,fill=white,densely dashed] (slots\slot) at (slots\slot pos) {};
\foreach \slot/\token in {A/a_1,B/a_2,C/b_1,D/b_2,E/b_3,F/c_1}
    \node[tokencircle,draw=replayink!60,fill=replayink!7,text=replayink] (replay\slot) at (replay\slot pos) {$\token^{\prime}$};
\foreach \tree in {plan,slots,replay} {
    \draw[treeedge] (\tree Root) -- (\tree A);
    \draw[treeedge] (\tree Root) -- (\tree B);
    \draw[treeedge] (\tree A) -- (\tree C);
    \draw[treeedge] (\tree A) -- (\tree D);
    \draw[treeedge] (\tree B) -- (\tree E);
    \draw[treeedge] (\tree C) -- (\tree F);
}
\node[annotation,text=planink] (planFooter) at (2.25,1.05) {Deterministic candidate tree};
\node[annotation] (slotsFooter) at (7.45,1.05) {Fixed candidate slots};
\node[annotation,text=replayink] (replayFooter) at (12.65,1.05) {Sampling and verification};
\foreach \tree/\horizontal in {plan/2.25,slots/7.45,replay/12.65} {
    \node[headingframe,minimum width=4.3cm,minimum height=5.44cm,inner sep=0pt] (\tree Frame) at (\horizontal,3.8) {};
}
\draw[transfer,shorten <=2pt,shorten >=2pt] (planFrame.east) -- (slotsFrame.west);
\draw[transfer,shorten <=2pt,shorten >=2pt] (slotsFrame.east) -- (replayFrame.west);
\end{tikzpicture}%
}
\endgroup
\caption{Overview of Tsubame. The first pass constructs a deterministic, context-aware candidate tree (left). Tsubame discards the planning tokens and retains the ordered topology (middle). The second pass regenerates candidates in the fixed slots using the chosen sampling scheme to form the tree used for target-model verification (right). Primed labels denote replayed candidates.}
\label{fig:tsubame-overview}
\end{figure}
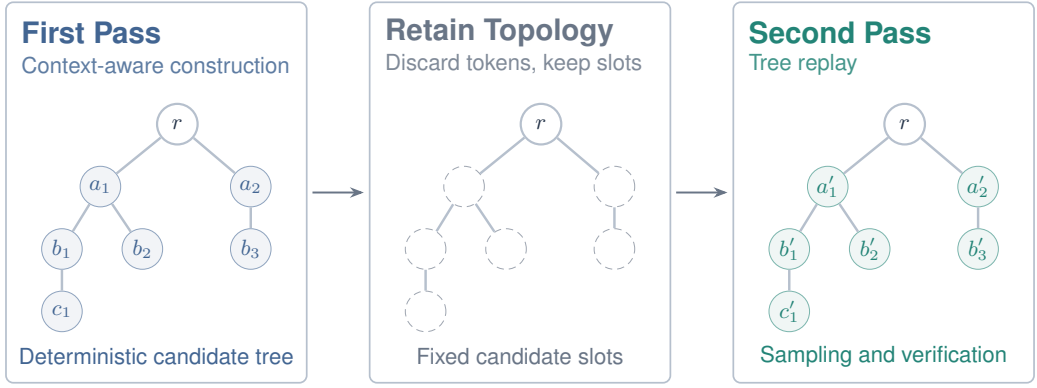

\subsection{First Pass: Tree Construction}
\label{sec:tree-construction}

The first pass allocates the candidate budget using draft path scores, adapting tree construction to how each drafter obtains its conditional distributions. For DFlash2, transition scores between candidates are precomputed, so conditional distributions can be obtained by table lookup, and we use node-wise max-heap expansion. Domino and DSpark instead require a lightweight correction step for each new prefix. For these drafters, we use EAGLE-2-style expand-then-prune, which batches corrections across frontier nodes rather than invoking them separately for each expanded node. In principle, Tsubame does not constrain how the first pass is built: any planner that produces a topology before replay can be used, chosen according to the properties of each drafter. Our choices favor efficiency; for example, batched expand-then-prune achieves higher throughput than node-wise max-heap expansion at larger budgets with similar acceptance length (Appendix~\ref{app:planner-comparison}).

Branching increases the number of prefixes requiring conditional distributions, making full-vocabulary logit correction across branches expensive. We therefore restrict the draft vocabulary to a pool of top-scoring tokens from the backbone at each position (Appendix~\ref{app:experimental-details}). Logit corrections are computed only within this pool, and the corrected logits are normalized over the pool to obtain the draft distribution used for candidate selection.

The first pass produces a deterministic, context-aware dynamic tree, which also serves as our deterministic baseline (First) when verified directly. Tsubame discards the planning tokens and retains only the topology $\mathcal{T}$, consisting of the parent-child relations and child ordering.

\FloatBarrier
\subsection[Second Pass: The ``Kaeshi'' Replay]{Second Pass: The ``Kaeshi'' Replay\protect\footnotemark}
\label{sec:replay-verification}
\footnotetext{The name alludes to \emph{Tsubame-gaeshi}. Here it evokes a second pass over the planned tree.}

The second pass fills $\mathcal{T}$ layer by layer. Let $x$ denote the current generation prefix and $y_{\operatorname{path}(u)}$ the replayed tokens along the path from the root to node $u$, excluding the root itself. This sequence is empty when $u$ is the root. For each parent $u$, we compute
\begin{equation}
    q_u(\cdot) = q\bigl(\cdot \mid x, y_{\operatorname{path}(u)}\bigr)
    \label{eq:replay-distribution}
\end{equation}
and populate its fixed child slots using the chosen sampling scheme. Each $q_u$ is obtained by invoking the drafter's own conditional correction under the replayed prefix, restricted to the same sparse candidate pools as in the first pass, with parents at the same depth processed in parallel. The topology remains unchanged throughout replay.

Replay thus differs from the first pass only in candidate selection: it samples new candidates under the chosen scheme rather than retaining the planning tokens. We mainly use sampling without replacement (WoR) and mixed sampling, paired with the verifiers described in Appendix~\ref{app:sampling-verification}.

\subsection{Losslessness Guarantee}
\label{sec:losslessness}

Tsubame inherits its losslessness from the static-tree verification algorithm used in the second pass. Appendix~\ref{app:sampling-verification} describes the compatible sampling--verification pairs.

\begin{proposition}[Losslessness of tree replay]
\label{prop:replay-losslessness}
Suppose that the first pass fixes the tree topology before replay and that the topology is not subsequently changed according to the replayed tokens. If the sampling--verification method used during replay is lossless for any fixed topology, then Tsubame preserves the target model's generation distribution.
\end{proposition}

\begin{proof}
    Fix the current decoding history and the topology produced by the first pass. Replay and verification then form a static-tree speculative decoding procedure on the fixed topology. Its output therefore follows the target model's distribution by the losslessness of the chosen verifier. Since this holds for every topology produced by the first pass, marginalizing over the topology and repeating the argument across decoding rounds preserves the target sequence distribution.
\end{proof}

In particular, Mixed/UniVer is lossless because UniVer provides this guarantee for any fixed tree topology under its compatible sampling rule \citep{weng2026univer}. During verification, however, the probabilities must correspond to the candidates generated during replay, not to the planning tokens discarded after the first pass. If replay uses a restricted candidate pool, these probabilities are normalized within that pool.

\section{Experiments}
\label{sec:experiments}

We evaluate whether replay improves acceptance and throughput over the deterministic first-pass tree, whether it outperforms sampled chains equipped with advanced verification, and how these effects vary with the candidate budget, the sampling--verification pair, and concurrent serving.

\subsection{Experimental Setup}
\label{sec:experimental-setup}

\paragraph{Models and datasets.}
Our main experiments use DFlash2 \citep{inco2026dflash2} and Domino \citep{huang2026domino}, with DSpark \citep{cheng2026dspark} results provided in Appendix~\ref{app:dspark-results}. Domino and DSpark use Qwen3-8B \citep{yang2025qwen3} as the default target model with thinking disabled, while DFlash2 uses Muse Glimmer 30B \citep{meta2026museglimmer}. Additional DFlash2 results with Qwen3.8-27B \citep{qwen2026qwen38} are reported in Appendix~\ref{app:dspark-results}. We evaluate on six datasets spanning three tasks: math (GSM8K \citep{cobbe2021gsm8k} and Math500 \citep{lightman2023verify}), code (HumanEval \citep{chen2021humaneval} and MBPP \citep{austin2021mbpp}), and conversation (MT-Bench \citep{zheng2023mtbench} and Alpaca \citep{taori2023alpaca}).

\paragraph{Proposal and verification strategies.}
For chains, we compare deterministic top-1 proposals with target-sample matching (Top1), and sampled proposals with rejection sampling (RS) \citep{leviathan2023fast,chen2023speculative} or Block Verification \citep{sun2025block}. For trees, we mainly use two sampling schemes: standard sampling without replacement (WoR), paired with layer-wise RRS or Traversal Verification \citep{chen2024sequoia,weng2025traversal}; and mixed sampling \citep{hu2025optimal}, paired with UniVer \citep{weng2026univer}. Table~\ref{tab:verifier-guide} summarizes their differences; please refer to Appendix \ref{app:sampling-verification} for detailed discussion on sampling and verification strategies.

\begin{table}[!t]
\centering
\small
\setlength{\tabcolsep}{4pt}
\renewcommand{\arraystretch}{1.12}
\caption{Guide to the verifiers used in our experiments. Scope distinguishes step-wise from sequence-level verification. The sampling column lists the configurations evaluated here, rather than all supported schemes. All pairs preserve the target distribution.}
\label{tab:verifier-guide}
\resizebox{\textwidth}{!}{%
\begin{tabular}{lllll}
\toprule
Structure & Verifier & Scope & Sampling & Description \\
\midrule
\multirow{3}{*}{Chain} & Naive & - & Top1 & Target-sample matching \\
 & RS & Step-wise & Random & Standard rejection sampling \\
 & BlockV & Sequence-level & Random & Improves RS through sequence-level coupling \\
\midrule
\multirow{4}{*}{Tree} & Naive & - & TopK & Target-sample matching \\
 & RRS & Step-wise & WoR & Standard recursive rejection across siblings \\
 & Traversal & Sequence-level & WoR & RRS with sequence-level coupling \\
 & UniVer & Sequence-level & Mixed & Unified verification supporting flexible sampling strategies \\
\bottomrule
\end{tabular}
}
\end{table}

\paragraph{Metrics and implementations.}
We report speedup ratio (SR) relative to autoregressive decoding and average acceptance length ($\tau$), with calculation details in Appendix~\ref{app:experimental-details}. Unless otherwise specified, we use temperature $T=1$, a maximum generation length of 512 tokens, and a candidate budget of $B=16$ nodes including the root, which also sets the draft length for chains. Single-request Qwen3-8B experiments run on NVIDIA GeForce RTX 3090 GPUs, while serving experiments and all Muse Glimmer 30B and Qwen3.8-27B experiments use NVIDIA RTX PRO 6000 GPUs with 96\,GB of memory. Unless otherwise noted, we pool results over three random seeds. All tree-construction and verification implementations use CUDA Graphs.

\subsection{Main Results}
\label{sec:main-results}

\begin{table}[t]
\centering
\footnotesize
\setlength{\tabcolsep}{2.5pt}
\renewcommand{\arraystretch}{1.12}
\caption{DFlash2 with Muse Glimmer 30B results on all six datasets at $B=16$. TopK denotes deterministic candidate selection in the first pass, and Naive denotes target-sample matching. Bold marks the best value in each column. The underlined verifier denotes the default configuration adopted by DFlash2.}
\label{tab:main-results-dflash2}
\resizebox{\textwidth}{!}{%
\begin{tabular}{lll*{6}{cc}}
\toprule
 & & & \multicolumn{4}{c}{Math} & \multicolumn{4}{c}{Code} & \multicolumn{4}{c}{Conversation} \\
\cmidrule(lr){4-7}\cmidrule(lr){8-11}\cmidrule(lr){12-15}
 & &  & \multicolumn{2}{c}{GSM8K} & \multicolumn{2}{c}{Math500} & \multicolumn{2}{c}{HumanEval} & \multicolumn{2}{c}{MBPP} & \multicolumn{2}{c}{MT-Bench} & \multicolumn{2}{c}{Alpaca} \\
\cmidrule(lr){4-5}\cmidrule(lr){6-7}\cmidrule(lr){8-9}\cmidrule(lr){10-11}\cmidrule(lr){12-13}\cmidrule(lr){14-15}
Structure & Sampling & Verifier & $\tau$ & SR & $\tau$ & SR & $\tau$ & SR & $\tau$ & SR & $\tau$ & SR & $\tau$ & SR \\
\midrule
\multirow{3}{*}{Chain} & Top1 & Naive & 5.44 & $4.15\times$ & 5.42 & $4.14\times$ & 4.41 & $3.35\times$ & 3.75 & $2.89\times$ & 3.38 & $2.55\times$ & 2.97 & $2.29\times$ \\
 & Random & \underline{RS} & 6.25 & $4.67\times$ & 6.12 & $4.58\times$ & 5.48 & $4.08\times$ & 4.82 & $3.63\times$ & 3.94 & $2.91\times$ & 3.62 & $2.73\times$ \\
 & Random & BlockV & 6.52 & $4.87\times$ & 6.25 & $4.67\times$ & 5.67 & $4.22\times$ & 5.07 & $3.82\times$ & 4.12 & $3.04\times$ & 3.70 & $2.79\times$ \\
\midrule
Tree (First) & TopK & Naive & 6.39 & $4.75\times$ & 6.21 & $4.62\times$ & 5.42 & $4.02\times$ & 4.69 & $3.52\times$ & 4.14 & $3.04\times$ & 3.70 & $2.78\times$ \\
\midrule
\multirow{3}{*}{Tree (Replay)} & WoR & RRS & 6.67 & $4.94\times$ & 6.59 & $4.89\times$ & 5.85 & $4.32\times$ & 5.12 & $3.83\times$ & 4.39 & $3.21\times$ & 3.96 & $2.97\times$ \\
 & WoR & Traversal & \textbf{6.83} & $\mathbf{5.06}\times$ & 6.68 & $4.96\times$ & \textbf{5.97} & $\mathbf{4.42}\times$ & \textbf{5.30} & $\mathbf{3.97}\times$ & \textbf{4.50} & $\mathbf{3.30}\times$ & \textbf{4.03} & $\mathbf{3.02}\times$ \\
 & Mixed & UniVer & 6.79 & $5.03\times$ & \textbf{6.70} & $\mathbf{4.97}\times$ & 5.96 & $4.40\times$ & 5.25 & $3.92\times$ & 4.46 & $3.26\times$ & 4.02 & $3.01\times$ \\
\bottomrule
\end{tabular}
}
\end{table}

\begin{table}[t]
\centering
\footnotesize
\setlength{\tabcolsep}{2.5pt}
\renewcommand{\arraystretch}{1.12}
\caption{Domino with Qwen3-8B results on all six datasets at $B=16$. All methods use sparse64 draft candidate pools. TopK denotes deterministic candidate selection in the first pass, and Naive denotes target-sample matching. Bold marks the best value in each column. The underlined verifier denotes the default configuration adopted by Domino.}
\label{tab:main-results}
\resizebox{\textwidth}{!}{%
\begin{tabular}{lll*{6}{cc}}
\toprule
 & & & \multicolumn{4}{c}{Math} & \multicolumn{4}{c}{Code} & \multicolumn{4}{c}{Conversation} \\
\cmidrule(lr){4-7}\cmidrule(lr){8-11}\cmidrule(lr){12-15}
 & &  & \multicolumn{2}{c}{GSM8K} & \multicolumn{2}{c}{Math500} & \multicolumn{2}{c}{HumanEval} & \multicolumn{2}{c}{MBPP} & \multicolumn{2}{c}{MT-Bench} & \multicolumn{2}{c}{Alpaca} \\
\cmidrule(lr){4-5}\cmidrule(lr){6-7}\cmidrule(lr){8-9}\cmidrule(lr){10-11}\cmidrule(lr){12-13}\cmidrule(lr){14-15}
Structure & Sampling & Verifier & $\tau$ & SR & $\tau$ & SR & $\tau$ & SR & $\tau$ & SR & $\tau$ & SR & $\tau$ & SR \\
\midrule
\multirow{3}{*}{Chain} & Top1 & \underline{Naive} & 7.21 & $5.54\times$ & 5.96 & $4.62\times$ & 5.49 & $4.26\times$ & 5.16 & $4.08\times$ & 3.14 & $2.44\times$ & 2.87 & $2.22\times$ \\
 & Random & RS & 8.05 & $6.14\times$ & 6.58 & $5.06\times$ & 6.08 & $4.68\times$ & 5.66 & $4.45\times$ & 3.56 & $2.75\times$ & 3.20 & $2.47\times$ \\
 & Random & BlockV & 8.38 & $6.38\times$ & 6.98 & $5.35\times$ & 6.30 & $4.83\times$ & 5.87 & $4.60\times$ & 3.73 & $2.87\times$ & 3.35 & $2.57\times$ \\
\midrule
Tree (First) & TopK & Naive & 8.00 & $6.04\times$ & 6.75 & $5.16\times$ & 6.49 & $4.97\times$ & 6.30 & $4.93\times$ & 4.26 & $3.24\times$ & 3.97 & $3.04\times$ \\
\midrule
\multirow{3}{*}{Tree (Replay)} & WoR & RRS & 8.37 & $6.31\times$ & 7.19 & $5.46\times$ & 6.57 & $4.98\times$ & 6.23 & $4.84\times$ & 4.22 & $3.19\times$ & 3.85 & $2.93\times$ \\
 & WoR & Traversal & \textbf{8.59} & $\mathbf{6.45}\times$ & 7.23 & $5.49\times$ & \textbf{6.88} & $\mathbf{5.21}\times$ & 6.44 & $4.99\times$ & 4.35 & $3.29\times$ & 3.96 & $3.01\times$ \\
 & Mixed & UniVer & 8.57 & $6.44\times$ & \textbf{7.26} & $\mathbf{5.51}\times$ & 6.87 & $5.19\times$ & \textbf{6.48} & $\mathbf{5.02}\times$ & \textbf{4.46} & $\mathbf{3.36}\times$ & \textbf{4.10} & $\mathbf{3.11}\times$ \\
\bottomrule
\end{tabular}
}
\end{table}

Tables~\ref{tab:main-results-dflash2} and~\ref{tab:main-results} distinguish the benefit of tree structure from the effects of candidate generation and verification at the same budget, $B=16$. For both drafters, deterministic dynamic trees achieve longer acceptance lengths than Top1 chains on every dataset, showing a clear structural advantage over a single deterministic continuation.

However, this advantage shrinks when chains use random sampling with RS or Block Verification. DFlash2 exhibits the clearest reversal: Block Verification exceeds the deterministic tree in throughput on five of the six datasets, with MT-Bench essentially tied, and in acceptance length on four. Domino shows the same limitation on the two math tasks and, more mildly, on HumanEval: on HumanEval, the deterministic tree's acceptance advantage narrows from 6.49 versus 5.49 against Top1 to 6.49 versus 6.30 against Block Verification, while on GSM8K and Math500 the ordering reverses. Comparing only against Top1 thus overstates the advantage of deterministic trees; Block Verification, which is optimal in single-chain speculative decoding, provides a more informative reference.

These comparisons suggest complementary roles: deterministic trees improve structural allocation, while stochastic candidates paired with suitable verification offer additional acceptance gains. Replay combines the two by retaining the planned topology while regenerating its candidates. With Mixed/UniVer, replay improves both acceptance length and throughput over the first pass and Block Verification on all six datasets for both drafters, reversing every deficit of the deterministic tree against Block Verification; on Domino, for example, the throughput deficits on GSM8K and Math500 become gains of 1.00\% and 2.93\%. The pooled throughput gains over the first pass are 8.51\% for DFlash2 and 4.19\% for Domino, showing that the additional acceptance more than offsets replay's extra computation. WoR/Traversal achieves comparable overall performance.

\subsection{Budget Scaling and Verification Choices}
\label{sec:budget-results}

\begin{table}[t]
\centering
\footnotesize
\setlength{\tabcolsep}{2.5pt}
\renewcommand{\arraystretch}{1.12}
\caption{DFlash2 with Muse Glimmer 30B across candidate budgets on all six datasets. Dataset cells report $\tau$ / SR. Replay uses WoR/RRS, WoR/Traversal, or Mixed/UniVer. The final column reports the respective changes relative to First pass at the same budget. Bold marks the best value for each metric within each budget. The percentages are calculated using the values before rounding.}
\label{tab:budget-results-dflash2}
\resizebox{\textwidth}{!}{%
\begin{tabular}{ll*{7}{c}}
\toprule
 & & \multicolumn{2}{c}{Math} & \multicolumn{2}{c}{Code} & \multicolumn{2}{c}{Conversation} & \\
\cmidrule(lr){3-4}\cmidrule(lr){5-6}\cmidrule(lr){7-8}
Budget & Method & GSM8K & Math500 & HumanEval & MBPP & MT-Bench & Alpaca & Gain $\tau$ / SR \\
\midrule
\multirow{4}{*}{16} & First & 6.39 / $4.75\times$ & 6.21 / $4.62\times$ & 5.42 / $4.02\times$ & 4.69 / $3.52\times$ & 4.14 / $3.04\times$ & 3.70 / $2.78\times$ & -- \\
 & Replay (RRS) & 6.67 / $4.94\times$ & 6.59 / $4.89\times$ & 5.85 / $4.32\times$ & 5.12 / $3.83\times$ & 4.39 / $3.21\times$ & 3.96 / $2.97\times$ & +6.97\% / +6.73\% \\
 & Replay (Traversal) & \textbf{6.83} / $\mathbf{5.06}\times$ & 6.68 / $4.96\times$ & \textbf{5.97} / $\mathbf{4.42}\times$ & \textbf{5.30} / $\mathbf{3.97}\times$ & \textbf{4.50} / $\mathbf{3.30}\times$ & \textbf{4.03} / $\mathbf{3.02}\times$ & \textbf{+9.45\%} / \textbf{+9.20\%} \\
 & Replay (UniVer) & 6.79 / $5.03\times$ & \textbf{6.70} / $\mathbf{4.97}\times$ & 5.96 / $4.40\times$ & 5.25 / $3.92\times$ & 4.46 / $3.26\times$ & 4.02 / $3.01\times$ & +8.88\% / +8.51\% \\
\midrule
\multirow{4}{*}{32} & First & 7.07 / $4.96\times$ & 7.04 / $4.94\times$ & 6.00 / $4.20\times$ & 5.28 / $3.74\times$ & 4.60 / $3.18\times$ & 4.09 / $2.90\times$ & -- \\
 & Replay (RRS) & 7.33 / $5.14\times$ & 7.20 / $5.05\times$ & 6.36 / $4.44\times$ & 5.66 / $4.00\times$ & 4.74 / $3.28\times$ & 4.27 / $3.03\times$ & +4.50\% / +4.41\% \\
 & Replay (Traversal) & 7.49 / $5.25\times$ & 7.38 / $5.18\times$ & 6.52 / $4.56\times$ & \textbf{5.82} / $\mathbf{4.11}\times$ & \textbf{4.92} / $\mathbf{3.40}\times$ & \textbf{4.42} / $\mathbf{3.13}\times$ & +7.67\% / +7.53\% \\
 & Replay (UniVer) & \textbf{7.54} / $\mathbf{5.28}\times$ & \textbf{7.52} / $\mathbf{5.26}\times$ & \textbf{6.57} / $\mathbf{4.58}\times$ & 5.80 / $4.10\times$ & 4.91 / $3.39\times$ & \textbf{4.42} / $\mathbf{3.13}\times$ & \textbf{+7.95\%} / \textbf{+7.68\%} \\
\midrule
\multirow{4}{*}{64} & First & 7.77 / $5.32\times$ & 7.67 / $5.26\times$ & 6.56 / $4.48\times$ & 5.75 / $3.97\times$ & 4.94 / $3.34\times$ & 4.42 / $3.06\times$ & -- \\
 & Replay (RRS) & 7.86 / $5.38\times$ & 7.69 / $5.27\times$ & 6.85 / $4.67\times$ & 6.02 / $4.16\times$ & 5.06 / $3.42\times$ & 4.56 / $3.15\times$ & +2.93\% / +2.81\% \\
 & Replay (Traversal) & 7.90 / $5.41\times$ & 7.81 / $5.35\times$ & 6.94 / $4.74\times$ & 6.22 / $4.30\times$ & 5.12 / $3.46\times$ & 4.64 / $3.22\times$ & +4.72\% / +4.67\% \\
 & Replay (UniVer) & \textbf{8.14} / $\mathbf{5.58}\times$ & \textbf{8.06} / $\mathbf{5.53}\times$ & \textbf{6.95} / $\mathbf{4.75}\times$ & \textbf{6.23} / $\mathbf{4.31}\times$ & \textbf{5.26} / $\mathbf{3.56}\times$ & \textbf{4.71} / $\mathbf{3.27}\times$ & \textbf{+6.36\%} / \textbf{+6.48\%} \\
\bottomrule
\end{tabular}
}
\end{table}

\begin{table}[t]
\centering
\footnotesize
\setlength{\tabcolsep}{2.5pt}
\renewcommand{\arraystretch}{1.12}
\caption{Domino with Qwen3-8B across candidate budgets on all six datasets. Dataset cells report $\tau$ / SR. Replay uses WoR/RRS, WoR/Traversal, or Mixed/UniVer. The final column reports the respective $\tau$ / SR changes relative to First pass at the same budget. Bold marks the best value for each metric within each budget. The percentages are calculated using the values before rounding.}
\label{tab:budget-results}
\resizebox{\textwidth}{!}{%
\begin{tabular}{ll*{7}{c}}
\toprule
 & & \multicolumn{2}{c}{Math} & \multicolumn{2}{c}{Code} & \multicolumn{2}{c}{Conversation} & \\
\cmidrule(lr){3-4}\cmidrule(lr){5-6}\cmidrule(lr){7-8}
Budget & Method & GSM8K & Math500 & HumanEval & MBPP & MT-Bench & Alpaca & Gain $\tau$ / SR \\
\midrule
\multirow{4}{*}{16} & First & 8.00 / $6.04\times$ & 6.75 / $5.16\times$ & 6.49 / $4.97\times$ & 6.30 / $4.93\times$ & 4.26 / $3.24\times$ & 3.97 / $3.04\times$ & -- \\
 & Replay (RRS) & 8.37 / $6.31\times$ & 7.19 / $5.46\times$ & 6.57 / $4.98\times$ & 6.23 / $4.84\times$ & 4.22 / $3.19\times$ & 3.85 / $2.93\times$ & +0.71\% / +0.09\% \\
 & Replay (Traversal) & \textbf{8.59} / $\mathbf{6.45}\times$ & 7.23 / $5.49\times$ & \textbf{6.88} / $\mathbf{5.21}\times$ & 6.44 / $4.99\times$ & 4.35 / $3.29\times$ & 3.96 / $3.01\times$ & +3.75\% / +2.98\% \\
 & Replay (UniVer) & 8.57 / $6.44\times$ & \textbf{7.26} / $\mathbf{5.51}\times$ & 6.87 / $5.19\times$ & \textbf{6.48} / $\mathbf{5.02}\times$ & \textbf{4.46} / $\mathbf{3.36}\times$ & \textbf{4.10} / $\mathbf{3.11}\times$ & \textbf{+5.10\%} / \textbf{+4.19\%} \\
\midrule
\multirow{4}{*}{32} & First & 9.36 / $6.18\times$ & 7.79 / $5.22\times$ & 7.35 / $4.93\times$ & 7.30 / $5.00\times$ & 4.78 / $3.20\times$ & 4.43 / $2.97\times$ & -- \\
 & Replay (RRS) & 9.53 / $6.28\times$ & 7.88 / $5.24\times$ & 7.44 / $4.95\times$ & 7.11 / $4.83\times$ & 4.69 / $3.12\times$ & 4.29 / $2.85\times$ & -0.74\% / -1.44\% \\
 & Replay (Traversal) & \textbf{9.72} / $\mathbf{6.39}\times$ & 8.16 / $5.41\times$ & 7.67 / $5.10\times$ & 7.24 / $4.91\times$ & 4.83 / $3.20\times$ & 4.38 / $2.91\times$ & +1.97\% / +1.10\% \\
 & Replay (UniVer) & 9.70 / $6.33\times$ & \textbf{8.21} / $\mathbf{5.44}\times$ & \textbf{7.80} / $\mathbf{5.17}\times$ & \textbf{7.59} / $\mathbf{5.15}\times$ & \textbf{5.01} / $\mathbf{3.32}\times$ & \textbf{4.62} / $\mathbf{3.06}\times$ & \textbf{+4.90\%} / \textbf{+3.80\%} \\
\midrule
\multirow{4}{*}{64} & First & 10.32 / $6.81\times$ & 8.46 / $5.67\times$ & 8.14 / $5.45\times$ & 8.05 / $5.51\times$ & 5.21 / $3.48\times$ & 4.84 / $3.25\times$ & -- \\
 & Replay (RRS) & 10.16 / $6.69\times$ & 8.46 / $5.62\times$ & 8.04 / $5.35\times$ & 7.73 / $5.26\times$ & 5.08 / $3.37\times$ & 4.63 / $3.08\times$ & -2.21\% / -2.84\% \\
 & Replay (Traversal) & 10.39 / $6.83\times$ & 8.79 / $5.83\times$ & 8.31 / $5.51\times$ & 8.06 / $5.47\times$ & 5.23 / $3.46\times$ & 4.79 / $3.18\times$ & +0.99\% / +0.03\% \\
 & Replay (UniVer) & \textbf{10.62} / $\mathbf{6.93}\times$ & \textbf{8.85} / $\mathbf{5.87}\times$ & \textbf{8.58} / $\mathbf{5.69}\times$ & \textbf{8.29} / $\mathbf{5.62}\times$ & \textbf{5.44} / $\mathbf{3.60}\times$ & \textbf{5.11} / $\mathbf{3.39}\times$ & \textbf{+4.60\%} / \textbf{+3.53\%} \\
\bottomrule
\end{tabular}
}
\end{table}

\paragraph{Replay vs. deterministic drafting.}
Increasing the candidate budget strengthens the deterministic tree but does not remove the benefit of replay. Tables~\ref{tab:budget-results-dflash2} and~\ref{tab:budget-results} show that the first pass achieves longer acceptance lengths as the budget grows from 16 to 64, while Mixed/UniVer improves both acceptance length and throughput on every dataset at each budget for both drafters. Its relative gains narrow as the deterministic baseline becomes stronger, yet at $B=64$ it still improves pooled throughput by 6.48\% for DFlash2 and 3.53\% for Domino. Allocating more nodes and improving how those nodes are populated and verified are therefore complementary.

\paragraph{Advanced vs. vanilla verification.}
The comparison between layer-wise RRS and Traversal clarifies the role of verification. Both use WoR candidates, but Traversal achieves higher pooled acceptance lengths and throughput at every budget for both drafters. In DFlash2, both methods' relative gains narrow as the budget grows, but Traversal retains a larger fraction of its $B=16$ acceptance gain at $B=64$. In Domino, the RRS acceptance gain becomes negative as the budget increases, whereas Traversal's remains positive. Resampling candidates alone thus does not fully exploit the multi-step continuations in the planned tree; a verifier that accounts for sequence-level dependencies makes more effective use of the same budget.

\paragraph{Impact of sampling strategy.}
Under WoR sampling, UniVer is mathematically equivalent to Traversal Verification \citep{weng2026univer}, so WoR/Traversal and Mixed/UniVer compare sampling strategies within the same verification framework. For DFlash2, WoR/Traversal is slightly stronger at $B=16$, while Mixed/UniVer is strongest at $B=32$ and $B=64$. For Domino, Mixed/UniVer achieves the highest pooled acceptance length and throughput at all three budgets, with its acceptance advantage over WoR/Traversal widening as the budget grows. These results motivate choosing the sampling strategy according to the task, model, and candidate budget (see also Appendix~\ref{app:dspark-results}).

\subsection{Serving Throughput}
\label{sec:serving-results}

\begin{table}[t]
\centering
\footnotesize
\setlength{\tabcolsep}{1.5pt}
\renewcommand{\arraystretch}{1.12}
\caption{DFlash2 with Muse Glimmer serving throughput (TPS) on a single RTX PRO 6000 GPU at $B=16$, pooled over GSM8K, HumanEval, and MT-Bench with three seeds. $C$ denotes concurrency, and parenthesized percentages are relative to the RS chain. The underlined verifier denotes the default chain verifier adopted by DFlash2. Bold marks the highest throughput in each column.}
\label{tab:serving-results}
\resizebox{\textwidth}{!}{%
\begin{tabular}{lll*{6}{c}}
\toprule
Structure & Sampling & Verifier & $C=1$ & $C=2$ & $C=4$ & $C=8$ & $C=16$ & $C=32$ \\
\midrule
\multirow{2}{*}{Chain} & Random & \underline{RS} & 114.7 & 216.3 & 418.8 & 723.8 & 1153.6 & 1449.8 \\
 & Random & BlockV & 118.9~(+3.6\%) & 224.6~(+3.9\%) & 434.2~(+3.7\%) & 744.9~(+2.9\%) & 1189.4~(+3.1\%) & 1504.5~(+3.8\%) \\
\midrule
Tree (First) & TopK & Naive & 112.9~(-1.5\%) & 215.1~(-0.5\%) & 414.7~(-1.0\%) & 711.8~(-1.7\%) & 1151.1~(-0.2\%) & 1451.2~(+0.1\%) \\
\midrule
\multirow{3}{*}{Tree (Replay)} & WoR & RRS & 120.7~(+5.2\%) & 227.1~(+5.0\%) & 439.4~(+4.9\%) & 755.3~(+4.3\%) & 1208.3~(+4.7\%) & 1527.5~(+5.4\%) \\
 & WoR & Traversal & \textbf{123.1~(+7.4\%)} & \textbf{232.4~(+7.5\%)} & \textbf{449.0~(+7.2\%)} & \textbf{771.6~(+6.6\%)} & 1230.2~(+6.6\%) & \textbf{1569.7~(+8.3\%)} \\
 & Mixed & UniVer & 122.2~(+6.6\%) & 231.7~(+7.1\%) & 446.2~(+6.5\%) & 767.7~(+6.1\%) & \textbf{1237.1~(+7.2\%)} & 1556.4~(+7.4\%) \\
\bottomrule
\end{tabular}
}
\end{table}

Concurrent serving tests whether replay remains beneficial when multiple requests share execution resources. We evaluate both drafters at $B=16$ on GSM8K, HumanEval, and MT-Bench at concurrencies up to 32 (Tables~\ref{tab:serving-results} and~\ref{tab:serving-results-domino}). For DFlash2, the deterministic tree stays close to the RS chain and below Block Verification at every concurrency, so structural allocation alone yields no throughput advantage. For Domino, the deterministic tree is already competitive with Block Verification, trailing it only at $C=8$.

In both settings, replay with Traversal or UniVer outperforms the deterministic tree and Block Verification at every concurrency. Relative to Block Verification, WoR/Traversal and Mixed/UniVer improve throughput by 3.4\%--4.3\% and 2.8\%--4.0\% for DFlash2, and by 1.3\%--3.0\% and 2.2\%--3.8\% for Domino. Replay's acceptance gains thus continue to offset its additional computation under concurrent serving, including when the deterministic tree is already competitive.

\begin{table}[t]
\centering
\footnotesize
\setlength{\tabcolsep}{1.5pt}
\renewcommand{\arraystretch}{1.12}
\caption{Domino with Qwen3-8B serving throughput (TPS) on a single RTX PRO 6000 GPU at $B=16$, pooled over GSM8K, HumanEval, and MT-Bench with three seeds. $C$ denotes concurrency, and parenthesized percentages are relative to the Top1 chain. The underlined verifier denotes the default chain verifier adopted by Domino. Bold marks the highest throughput in each column.}
\label{tab:serving-results-domino}
\resizebox{\textwidth}{!}{%
\begin{tabular}{lll*{5}{c}}
\toprule
Structure & Sampling & Verifier & $C=2$ & $C=4$ & $C=8$ & $C=16$ & $C=32$ \\
\midrule
\multirow{3}{*}{Chain} & Top1 & \underline{Naive} & 658.8 & 1253.2 & 2232.1 & 3407.3 & 4301.8 \\
 & Random & RS & 684.4~(+3.9\%) & 1275.5~(+1.8\%) & 2281.7~(+2.2\%) & 3530.4~(+3.6\%) & 4416.5~(+2.7\%) \\
 & Random & BlockV & 701.4~(+6.5\%) & 1322.7~(+5.5\%) & 2365.8~(+6.0\%) & 3629.0~(+6.5\%) & 4542.1~(+5.6\%) \\
\midrule
Tree (First) & TopK & Naive & 710.3~(+7.8\%) & 1326.6~(+5.9\%) & 2352.3~(+5.4\%) & 3636.9~(+6.7\%) & 4544.0~(+5.6\%) \\
\midrule
\multirow{3}{*}{Tree (Replay)} & WoR & RRS & 700.5~(+6.3\%) & 1311.2~(+4.6\%) & 2368.0~(+6.1\%) & 3606.1~(+5.8\%) & 4544.5~(+5.6\%) \\
 & WoR & Traversal & 722.2~(+9.6\%) & 1346.9~(+7.5\%) & 2410.5~(+8.0\%) & 3677.9~(+7.9\%) & 4667.2~(+8.5\%) \\
 & Mixed & UniVer & \textbf{728.4~(+10.6\%)} & \textbf{1357.1~(+8.3\%)} & \textbf{2419.0~(+8.4\%)} & \textbf{3709.2~(+8.9\%)} & \textbf{4677.1~(+8.7\%)} \\
\bottomrule
\end{tabular}
}
\end{table}

\section{Conclusion}
\label{sec:conclusion}

This work examines the interplay between tree structure, candidate sampling, and verification in speculative decoding. Our analysis shows that the structural advantage of deterministic dynamic trees does not always compensate for the gains available to sampled chains with advanced verification. We introduce Tsubame to combine these complementary benefits: it first plans a context-aware topology, then replays that topology to generate stochastic candidates for verification. Diffusion-based drafters make this second pass inexpensive, while compatible verification guarantees that the target distribution is preserved. Experiments across DFlash2, Domino, and DSpark demonstrate improvements in acceptance length and throughput, including cases where replay reverses the disadvantage of deterministic trees against advanced chain baselines. These benefits extend to larger candidate budgets and concurrent serving. Together, our findings show that topology planning need not determine the final candidates: separating these decisions provides a practical way to combine structural allocation with effective sampling and verification.

\bibliography{iclr2027_conference}
\bibliographystyle{iclr2027_conference}

\appendix

\section{Additional Background on Sampling and Verification}
\label{app:sampling-verification}

Section~\ref{sec:sampling-verification} contrasts RS with top-1 target-sample matching. More generally, target-sample matching accepts a child from a fixed candidate set $C$ with probability $\sum_{v\in C}p(v)$. With multiple candidates and positions, both the sampling scheme and the verification rule provide further design choices.

With multiple candidates at one position, the sampling scheme is a separate design choice. Sampling with replacement draws independently from $q$, allowing repeated tokens. Sampling without replacement (WoR) excludes previously drawn tokens and renormalizes $q$ before each subsequent draw. Mixed sampling, corresponding to the greedy draft sampling scheme of \citet{hu2025optimal}, selects all but one candidate deterministically by draft probability and samples the last from the renormalized remainder. These schemes induce different joint candidate distributions and can have different optimal acceptance rates \citep{hu2025optimal}.

Each sampling scheme must be paired with a compatible verifier. With-replacement candidates can be verified using Recursive Rejection Sampling (RRS), as in SpecInfer \citep{miao2024specinfer}, or alternative rules such as K-Seq from SpecTr \citep{sun2023spectr}. WoR candidates are commonly paired with WoR variants of RRS, as in Sequoia, EAGLE, and Recursive Speculative Decoding (RSD) \citep{chen2024sequoia,li2024eagle,jeon2024rsd}. Even under the same candidate distribution, different verifiers can achieve different acceptance rates while all preserving the target distribution. For fixed target and draft distributions and a candidate budget, the sampling strategy determines the attainable acceptance ceiling, while the verifier determines how closely it is approached \citep{sun2023spectr,hu2025optimal}.

Beyond a single token or multiple candidates at one position, chains and trees introduce multi-step dependencies that verification must also account for. Block Verification shows that token-wise RS need not maximize expected token yield over multiple positions and uses sequence-level probabilities to make acceptance decisions, but is restricted to chains \citep{sun2025block}. Traversal Verification extends sequence-level verification to trees through sequence-level RRS with sampling without replacement \citep{weng2025traversal}. UniVer accommodates different sampling strategies within a unified verification framework and reduces exactly to Traversal Verification under WoR sampling \citep{weng2026univer}.

\clearpage
\section{Additional Experimental Details and Results}
\label{app:experimental-details}

\subsection{Evaluation Samples and Aggregation}

\paragraph{Evaluation samples.}
We follow Domino's per-dataset sample counts \citep{huang2026domino}, evaluating on six datasets: 128 GSM8K, 128 Math500, 164 HumanEval, 128 MBPP, 80 MT-Bench, and 128 Alpaca samples. A multiturn prompt is treated as one sample. Serving experiments use GSM8K, HumanEval, and MT-Bench with 3 seeds, generating 512 tokens per request.

\paragraph{Draft candidate pools.}
DFlash2 natively uses 16-token candidate pools, which we retain in the Muse experiments. For Domino, we introduce 64-token pools for both chains and trees; Appendix~\ref{app:domino-sparse-control} provides the corresponding full-vocabulary control. For DSpark, we introduce sparse candidate pools for trees using the dual-1024 configuration, while its native chain remains dense. These restrictions apply only to draft proposals, not to the target distribution.

\paragraph{Pooled and macro aggregation.}
Unless otherwise specified, all reported metrics use pooled aggregation across the samples and seeds in each comparison group. Throughput is total output tokens divided by total recorded time. Speedup ratio (SR) is the ratio of speculative to AR throughput for the same target model, dataset, hardware, and backend. Let $A_i$ and $R_i$ denote the sum of recorded acceptance lengths and the number of decoding rounds for sample $i$, treating each seed's evaluation as a separate observation. Pooled and sample-level macro acceptance lengths are defined as
\begin{equation}
    \tau_{\mathrm{pooled}} = \frac{\sum_{i=1}^{N} A_i}{\sum_{i=1}^{N} R_i},
    \qquad
    \tau_{\mathrm{macro}} = \frac{1}{N}\sum_{i=1}^{N}\frac{A_i}{R_i}.
\end{equation}
Pooled aggregation weights each sample by its number of decoding rounds, whereas macro aggregation gives each sample equal weight. The two statistics can therefore differ when samples have different round counts and acceptance lengths. Cross-dataset pooled results combine the underlying totals across datasets using the same rule.

\begin{table}[!htbp]
\centering
\small
\setlength{\tabcolsep}{2pt}
\renewcommand{\arraystretch}{1.12}
\caption{Pooled and sample-level macro acceptance lengths for each dataset, combining three seeds for the methods in Tables~\ref{tab:main-results-dflash2} and~\ref{tab:main-results}. Macro gives each sample--seed observation equal weight.}
\label{tab:pooled-macro}
\begin{tabular}{l*{6}{cc}}
\toprule
& \multicolumn{2}{c}{GSM8K} & \multicolumn{2}{c}{Math500} & \multicolumn{2}{c}{HumanEval} & \multicolumn{2}{c}{MBPP} & \multicolumn{2}{c}{MT-Bench} & \multicolumn{2}{c}{Alpaca} \\
\cmidrule(lr){2-3}\cmidrule(lr){4-5}\cmidrule(lr){6-7}\cmidrule(lr){8-9}\cmidrule(lr){10-11}\cmidrule(lr){12-13}
Method & Pooled & Macro & Pooled & Macro & Pooled & Macro & Pooled & Macro & Pooled & Macro & Pooled & Macro \\
\midrule
\multicolumn{13}{c}{\textbf{Domino with Qwen3-8B}} \\
\midrule
Top1 & 7.21 & 7.88 & 5.96 & 6.78 & 5.49 & 5.65 & 5.16 & 5.64 & 3.14 & 4.05 & 2.87 & 3.20 \\
RS & 8.05 & 8.55 & 6.58 & 7.32 & 6.08 & 6.27 & 5.66 & 6.06 & 3.56 & 4.36 & 3.20 & 3.43 \\
BlockV & 8.38 & 8.88 & 6.98 & 7.71 & 6.30 & 6.48 & 5.87 & 6.32 & 3.73 & 4.55 & 3.35 & 3.55 \\
First & 8.00 & 8.48 & 6.75 & 7.45 & 6.49 & 6.66 & 6.30 & 6.68 & 4.26 & 5.00 & 3.97 & 4.22 \\
WoR/RRS & 8.37 & 8.83 & 7.19 & 7.80 & 6.57 & 6.73 & 6.23 & 6.64 & 4.22 & 4.92 & 3.85 & 4.01 \\
WoR/Traversal & 8.59 & 9.03 & 7.23 & 7.83 & 6.88 & 7.04 & 6.44 & 6.85 & 4.35 & 5.12 & 3.96 & 4.12 \\
Mixed/UniVer & 8.57 & 9.09 & 7.26 & 7.84 & 6.87 & 7.06 & 6.48 & 6.86 & 4.46 & 5.14 & 4.10 & 4.27 \\
\midrule
\multicolumn{13}{c}{\textbf{DFlash2 with Muse Glimmer 30B}} \\
\midrule
Top1 & 5.44 & 5.83 & 5.42 & 5.82 & 4.41 & 4.55 & 3.75 & 3.82 & 3.38 & 3.69 & 2.97 & 3.18 \\
RS & 6.25 & 6.65 & 6.12 & 6.55 & 5.48 & 5.64 & 4.82 & 4.91 & 3.94 & 4.30 & 3.62 & 3.88 \\
BlockV & 6.52 & 6.88 & 6.25 & 6.71 & 5.67 & 5.83 & 5.07 & 5.17 & 4.12 & 4.47 & 3.70 & 4.01 \\
First & 6.39 & 6.74 & 6.21 & 6.58 & 5.42 & 5.54 & 4.69 & 4.76 & 4.14 & 4.44 & 3.70 & 3.94 \\
WoR/RRS & 6.67 & 6.96 & 6.59 & 6.95 & 5.85 & 5.96 & 5.12 & 5.19 & 4.39 & 4.70 & 3.96 & 4.19 \\
WoR/Traversal & 6.83 & 7.14 & 6.68 & 7.05 & 5.97 & 6.11 & 5.30 & 5.38 & 4.50 & 4.83 & 4.03 & 4.30 \\
Mixed/UniVer & 6.79 & 7.10 & 6.70 & 7.06 & 5.96 & 6.08 & 5.25 & 5.32 & 4.46 & 4.78 & 4.02 & 4.25 \\
\bottomrule
\end{tabular}
\end{table}

Table~\ref{tab:pooled-macro} compares both acceptance-length statistics separately for each dataset and drafter. Within each dataset, macro values sum the per-sample means across the three seeds and divide by the number of sample--seed observations. This breakdown shows how the choice of aggregation affects the reported acceptance length on individual datasets.

\subsection{Six-Dataset Aggregates and Cross-Seed Variability}
\label{app:seed-variability}

We aggregate the six-dataset results from Tables~\ref{tab:main-results-dflash2} and~\ref{tab:main-results} at $B=16$. For each method and seed, we pool the recorded totals across all six datasets to obtain acceptance length and throughput. SR is the ratio of pooled speculative to pooled AR throughput for the same six-dataset workload, with the AR reference held fixed across seeds. Table~\ref{tab:seed-variability} reports the all-seed pooled results alongside the mean and sample standard deviation (SD) of the three seed-level estimates.

\begin{table}[!htbp]
\centering
\footnotesize
\setlength{\tabcolsep}{3pt}
\renewcommand{\arraystretch}{1.12}
\caption{Six-dataset aggregates at $B=16$ for the methods in the two main tables. Pooled combines all datasets and seeds; seed mean $\pm$ SD and 95\% CIs use the three separately pooled seed-level estimates. SR is expressed as a multiplicative speedup.}
\label{tab:seed-variability}
\resizebox{\textwidth}{!}{%
\begin{tabular}{l*{6}{c}}
\toprule
& \multicolumn{3}{c}{$\tau$} & \multicolumn{3}{c}{SR ($\times$)} \\
\cmidrule(lr){2-4}\cmidrule(lr){5-7}
Method & Pooled & Seed mean $\pm$ SD & 95\% CI & Pooled & Seed mean $\pm$ SD & 95\% CI \\
\midrule
\multicolumn{7}{l}{\textbf{Domino with Qwen3-8B}} \\
Top1 & 4.539 & $4.539 \pm 0.009$ & [4.516, 4.563] & 3.524 & $3.524 \pm 0.007$ & [3.506, 3.542] \\
RS & 5.065 & $5.065 \pm 0.013$ & [5.032, 5.098] & 3.904 & $3.904 \pm 0.009$ & [3.882, 3.927] \\
BlockV & 5.291 & $5.291 \pm 0.018$ & [5.245, 5.336] & 4.066 & $4.066 \pm 0.013$ & [4.032, 4.100] \\
First & 5.676 & $5.676 \pm 0.015$ & [5.639, 5.714] & 4.339 & $4.339 \pm 0.016$ & [4.299, 4.380] \\
WoR/RRS & 5.716 & $5.716 \pm 0.016$ & [5.676, 5.757] & 4.343 & $4.343 \pm 0.012$ & [4.313, 4.373] \\
WoR/Traversal & 5.889 & $5.889 \pm 0.019$ & [5.841, 5.937] & 4.468 & $4.468 \pm 0.015$ & [4.430, 4.506] \\
Mixed/UniVer & 5.966 & $5.966 \pm 0.018$ & [5.920, 6.011] & 4.521 & $4.521 \pm 0.014$ & [4.487, 4.555] \\
\midrule
\multicolumn{7}{l}{\textbf{DFlash2 with Muse Glimmer 30B}} \\
Top1 & 3.995 & $3.995 \pm 0.012$ & [3.966, 4.024] & 3.046 & $3.047 \pm 0.007$ & [3.030, 3.063] \\
RS & 4.808 & $4.808 \pm 0.004$ & [4.798, 4.818] & 3.592 & $3.592 \pm 0.012$ & [3.562, 3.622] \\
BlockV & 4.986 & $4.986 \pm 0.017$ & [4.943, 5.028] & 3.723 & $3.723 \pm 0.024$ & [3.664, 3.783] \\
First & 4.871 & $4.871 \pm 0.016$ & [4.831, 4.912] & 3.620 & $3.620 \pm 0.012$ & [3.592, 3.649] \\
WoR/RRS & 5.211 & $5.211 \pm 0.011$ & [5.183, 5.239] & 3.864 & $3.864 \pm 0.006$ & [3.850, 3.878] \\
WoR/Traversal & 5.332 & $5.332 \pm 0.017$ & [5.289, 5.375] & 3.953 & $3.953 \pm 0.015$ & [3.915, 3.992] \\
Mixed/UniVer & 5.304 & $5.304 \pm 0.026$ & [5.239, 5.369] & 3.928 & $3.928 \pm 0.022$ & [3.873, 3.984] \\
\bottomrule
\end{tabular}
}
\end{table}

For a seed-level metric $m$, the approximate 95\% confidence interval (CI) for its mean is
\begin{equation}
    \bar{m} \pm t_{0.975,2}\frac{s_m}{\sqrt{3}},
    \qquad t_{0.975,2}\approx 4.303,
\end{equation}
where $s_m$ is the sample SD across seeds, computed with denominator $3-1$. These Student-$t$ intervals assume independent, normally distributed seed-level estimates. The six datasets form a fixed evaluation suite; the intervals characterize run-to-run uncertainty on this suite. CIs are centered on the seed mean, while the all-seed pooled result is computed directly from the combined totals.

\clearpage
\subsection{Choice of First-Pass Planner}
\label{app:planner-comparison}

Tsubame does not constrain how the first pass is built: any planner that fixes the topology before replay can be used. We use node-wise max-heap expansion for DFlash2 and EAGLE-2-style expand-then-prune for Domino and DSpark. Both allocate nodes by cumulative draft path probability and differ only in implementation. Max-heap search maximizes summed draft path probability under a node budget \citep{ringel2026ddtree,lin2026dominotree}, and expand-then-prune can recover the same solution given a sufficiently wide frontier and candidate pool. We simply retain EAGLE-2's default frontier width of 10; larger frontiers such as 16 or 32 can also be used, and we did not tune this setting.

Table~\ref{tab:planner-comparison} compares DominoTree with our first pass on Domino. The two planners achieve similar acceptance lengths: DominoTree is slightly ahead at $B=16$, consistent with our narrower frontier, and the difference largely vanishes at $B=32$ and $B=64$. Expand-then-prune batches GRU updates and logit corrections across each frontier, requiring at most one sequential correction stage per expansion level, whereas DominoTree's sequential stages scale with the number of expanded nodes \citep{lin2026dominotree}; our first pass therefore achieves higher throughput at larger budgets. Mixed/UniVer, built on our first pass, exceeds both planners at every budget, reflecting that Tsubame's gains come from replay rather than from the choice of planner.

\begin{table}[!htbp]
\centering
\footnotesize
\setlength{\tabcolsep}{2.5pt}
\renewcommand{\arraystretch}{1.12}
\caption{Domino with Qwen3-8B: DominoTree, First pass, and Mixed/UniVer under stochastic decoding ($T=1$) across candidate budgets. Budgets include the root, and results pool three seeds.}
\label{tab:planner-comparison}
\resizebox{\textwidth}{!}{%
\begin{tabular}{cl*{6}{cc}}
\toprule
& & \multicolumn{4}{c}{Math} & \multicolumn{4}{c}{Code} & \multicolumn{4}{c}{Conversation} \\
\cmidrule(lr){3-6}\cmidrule(lr){7-10}\cmidrule(lr){11-14}
& & \multicolumn{2}{c}{GSM8K} & \multicolumn{2}{c}{Math500} & \multicolumn{2}{c}{HumanEval} & \multicolumn{2}{c}{MBPP} & \multicolumn{2}{c}{MT-Bench} & \multicolumn{2}{c}{Alpaca} \\
\cmidrule(lr){3-4}\cmidrule(lr){5-6}\cmidrule(lr){7-8}\cmidrule(lr){9-10}\cmidrule(lr){11-12}\cmidrule(lr){13-14}
$B$ & Method & $\tau$ & SR & $\tau$ & SR & $\tau$ & SR & $\tau$ & SR & $\tau$ & SR & $\tau$ & SR \\
\midrule
\multirow{3}{*}{16} & DominoTree & 8.08 & $6.09\times$ & 6.81 & $5.18\times$ & 6.56 & $4.99\times$ & 6.30 & $4.89\times$ & 4.28 & $3.24\times$ & 4.00 & $3.04\times$ \\
& First & 8.00 & $6.04\times$ & 6.75 & $5.16\times$ & 6.49 & $4.97\times$ & 6.30 & $4.93\times$ & 4.26 & $3.24\times$ & 3.97 & $3.04\times$ \\
& Mixed/UniVer & \textbf{8.57} & $\mathbf{6.44}\times$ & \textbf{7.26} & $\mathbf{5.51}\times$ & \textbf{6.87} & $\mathbf{5.19}\times$ & \textbf{6.48} & $\mathbf{5.02}\times$ & \textbf{4.46} & $\mathbf{3.36}\times$ & \textbf{4.10} & $\mathbf{3.11}\times$ \\
\midrule
\multirow{3}{*}{32} & DominoTree & 9.41 & $6.10\times$ & 7.66 & $5.02\times$ & 7.43 & $4.87\times$ & 7.31 & $4.89\times$ & 4.79 & $3.13\times$ & 4.42 & $2.89\times$ \\
& First & 9.36 & $6.18\times$ & 7.79 & $5.22\times$ & 7.35 & $4.93\times$ & 7.30 & $5.00\times$ & 4.78 & $3.20\times$ & 4.43 & $2.97\times$ \\
& Mixed/UniVer & \textbf{9.70} & $\mathbf{6.33}\times$ & \textbf{8.21} & $\mathbf{5.44}\times$ & \textbf{7.80} & $\mathbf{5.17}\times$ & \textbf{7.59} & $\mathbf{5.15}\times$ & \textbf{5.01} & $\mathbf{3.32}\times$ & \textbf{4.62} & $\mathbf{3.06}\times$ \\
\midrule
\multirow{3}{*}{64} & DominoTree & 10.12 & $6.35\times$ & 8.46 & $5.35\times$ & 8.10 & $5.12\times$ & 8.06 & $5.20\times$ & 5.27 & $3.33\times$ & 4.90 & $3.10\times$ \\
& First & 10.32 & $6.81\times$ & 8.46 & $5.67\times$ & 8.14 & $5.45\times$ & 8.05 & $5.51\times$ & 5.21 & $3.48\times$ & 4.84 & $3.25\times$ \\
& Mixed/UniVer & \textbf{10.62} & $\mathbf{6.93}\times$ & \textbf{8.85} & $\mathbf{5.87}\times$ & \textbf{8.58} & $\mathbf{5.69}\times$ & \textbf{8.29} & $\mathbf{5.62}\times$ & \textbf{5.44} & $\mathbf{3.60}\times$ & \textbf{5.11} & $\mathbf{3.39}\times$ \\
\bottomrule
\end{tabular}
}
\end{table}

\clearpage
\subsection{Additional Experiments with DSpark and DFlash2--Qwen3.8}
\label{app:dspark-results}

\paragraph{DSpark with Qwen3-8B.}
Replay with Mixed/UniVer improves acceptance length and throughput over First pass on all six datasets (Table~\ref{tab:dspark-main}), with pooled gains of 4.95\% and 3.33\%, respectively.

\begin{table}[!htbp]
\centering
\footnotesize
\setlength{\tabcolsep}{2.5pt}
\renewcommand{\arraystretch}{1.12}
\caption{DSpark with Qwen3-8B results on all six datasets, pooled over three seeds. Chains retain DSpark's native draft depth of 8; trees use a candidate budget of $B=16$. TopK denotes deterministic candidate selection in the first pass, and Naive denotes target-sample matching. Bold marks the best value in each column.}
\label{tab:dspark-main}
\resizebox{\textwidth}{!}{%
\begin{tabular}{lll*{6}{cc}}
\toprule
 & & & \multicolumn{4}{c}{Math} & \multicolumn{4}{c}{Code} & \multicolumn{4}{c}{Conversation} \\
\cmidrule(lr){4-7}\cmidrule(lr){8-11}\cmidrule(lr){12-15}
 & &  & \multicolumn{2}{c}{GSM8K} & \multicolumn{2}{c}{Math500} & \multicolumn{2}{c}{HumanEval} & \multicolumn{2}{c}{MBPP} & \multicolumn{2}{c}{MT-Bench} & \multicolumn{2}{c}{Alpaca} \\
\cmidrule(lr){4-5}\cmidrule(lr){6-7}\cmidrule(lr){8-9}\cmidrule(lr){10-11}\cmidrule(lr){12-13}\cmidrule(lr){14-15}
Structure & Sampling & Verifier & $\tau$ & SR & $\tau$ & SR & $\tau$ & SR & $\tau$ & SR & $\tau$ & SR & $\tau$ & SR \\
\midrule
\multirow{3}{*}{Chain} & Top1 & Naive & 5.73 & $4.48\times$ & 5.14 & $4.04\times$ & 5.00 & $3.92\times$ & 4.82 & $3.83\times$ & 3.39 & $2.66\times$ & 3.24 & $2.49\times$ \\
 & Random & RS & 6.35 & $4.84\times$ & 5.92 & $4.56\times$ & 5.59 & $4.34\times$ & 5.28 & $4.13\times$ & 3.81 & $2.93\times$ & 3.54 & $2.68\times$ \\
 & Random & BlockV & 6.39 & $4.89\times$ & 6.00 & $4.65\times$ & 5.64 & $4.37\times$ & 5.42 & $4.21\times$ & 3.90 & $3.01\times$ & 3.63 & $2.76\times$ \\
\midrule
Tree (First) & TopK & Naive & 6.66 & $5.03\times$ & 6.12 & $4.63\times$ & 5.92 & $4.50\times$ & 5.72 & $4.43\times$ & 4.17 & $3.16\times$ & 3.99 & $3.02\times$ \\
\midrule
\multirow{3}{*}{Tree (Replay)} & WoR & RRS & 6.78 & $5.02\times$ & 6.30 & $4.72\times$ & 6.10 & $4.59\times$ & 5.87 & $4.48\times$ & 4.31 & $3.23\times$ & 4.07 & $3.03\times$ \\
 & WoR & Traversal & 6.83 & $5.08\times$ & \textbf{6.38} & $\mathbf{4.81}\times$ & 6.16 & $4.60\times$ & 5.94 & $\mathbf{4.55}\times$ & 4.41 & $3.30\times$ & 4.15 & $3.09\times$ \\
 & Mixed & UniVer & \textbf{6.85} & $\mathbf{5.09}\times$ & \textbf{6.38} & $4.74\times$ & \textbf{6.21} & $\mathbf{4.66}\times$ & \textbf{5.98} & $\mathbf{4.55}\times$ & \textbf{4.44} & $\mathbf{3.32}\times$ & \textbf{4.22} & $\mathbf{3.13}\times$ \\
\bottomrule
\end{tabular}
}
\end{table}

\paragraph{DFlash2 with Qwen3.8-27B.}
The released DFlash2 implementation for Qwen3.8-27B uses SGLang \citep{zheng2024sglang}, whereas our Muse experiments use the Python implementation. We therefore compare methods using DFlash2's native SGLang single-request configuration and a matched SGLang AR baseline, reporting request-mean acceptance length and SR. As shown in Table~\ref{tab:qwen38-b16}, Mixed/UniVer improves the six-dataset request-mean acceptance length by 1.84\% and throughput by 1.25\% over First pass. On code tasks, sampled chains with RS yield shorter acceptance lengths than Top1 chains (e.g., 4.26 vs.\ 4.40 on HumanEval), indicating that sampling a lone candidate is not beneficial for this model--task pair. We therefore fill singleton child slots with the top-1 token on code tasks, while nodes with multiple children retain mixed sampling (top-$(m-1)$ candidates plus one sampled candidate). This choice follows the chain-level comparison between RS and Top1, not tree-level results.

\begin{table}[!htbp]
\centering
\footnotesize
\setlength{\tabcolsep}{2.5pt}
\renewcommand{\arraystretch}{1.12}
\caption{DFlash2 with Qwen3.8-27B (non-thinking) results on all six datasets, pooled over three seeds. Chains retain Qwen3.8-DFlash2's native draft depth of 8; trees use a candidate budget of $B=16$. Results use the native SGLang single-request configuration and report request-mean acceptance length and SR relative to a matched SGLang AR baseline. TopK denotes deterministic candidate selection in the first pass, and Naive denotes target-sample matching. Bold marks the best value in each column.}
\label{tab:qwen38-b16}
\resizebox{\textwidth}{!}{%
\begin{tabular}{lll*{6}{cc}}
\toprule
 & & & \multicolumn{4}{c}{Math} & \multicolumn{4}{c}{Code} & \multicolumn{4}{c}{Conversation} \\
\cmidrule(lr){4-7}\cmidrule(lr){8-11}\cmidrule(lr){12-15}
 & & & \multicolumn{2}{c}{GSM8K} & \multicolumn{2}{c}{Math500} & \multicolumn{2}{c}{HumanEval} & \multicolumn{2}{c}{MBPP} & \multicolumn{2}{c}{MT-Bench} & \multicolumn{2}{c}{Alpaca} \\
\cmidrule(lr){4-5}\cmidrule(lr){6-7}\cmidrule(lr){8-9}\cmidrule(lr){10-11}\cmidrule(lr){12-13}\cmidrule(lr){14-15}
Structure & Sampling & Verifier & $\tau_{\mathrm{req}}$ & SR & $\tau_{\mathrm{req}}$ & SR & $\tau_{\mathrm{req}}$ & SR & $\tau_{\mathrm{req}}$ & SR & $\tau_{\mathrm{req}}$ & SR & $\tau_{\mathrm{req}}$ & SR \\
\midrule
\multirow{3}{*}{Chain} & Top1 & Naive & 4.26 & $3.15\times$ & 5.57 & $4.31\times$ & 4.40 & $3.21\times$ & 3.46 & $2.55\times$ & 3.40 & $2.49\times$ & 2.85 & $2.20\times$ \\
 & Random & RS & 4.39 & $3.22\times$ & 5.79 & $4.50\times$ & 4.26 & $2.96\times$ & 3.40 & $2.41\times$ & 3.70 & $2.71\times$ & 3.08 & $2.35\times$ \\
 & Random & BlockV & 4.44 & $3.26\times$ & 5.88 & $4.57\times$ & 4.37 & $3.04\times$ & 3.43 & $2.44\times$ & 3.73 & $2.74\times$ & 3.14 & $2.40\times$ \\
\midrule
Tree (First) & TopK & Naive & 4.81 & $3.38\times$ & 6.38 & $4.71\times$ & 4.84 & $3.35\times$ & 3.99 & $2.81\times$ & 4.11 & $2.88\times$ & 3.45 & $2.53\times$ \\
\midrule
\multirow{3}{*}{Tree (Replay)} & WoR & RRS & 4.80 & $3.38\times$ & 6.35 & $4.63\times$ & 4.80 & $3.23\times$ & 4.03 & $2.76\times$ & 4.15 & $2.91\times$ & 3.51 & $2.56\times$ \\
 & WoR & Traversal & 4.89 & $\mathbf{3.46}\times$ & 6.39 & $4.67\times$ & 4.80 & $3.24\times$ & 3.94 & $2.69\times$ & 4.17 & $2.93\times$ & 3.56 & $2.57\times$ \\
 & Mixed & UniVer & \textbf{4.91} & $3.42\times$ & \textbf{6.47} & $\mathbf{4.74}\times$ & \textbf{4.90} & $\mathbf{3.38}\times$ & \textbf{4.04} & $\mathbf{2.82}\times$ & \textbf{4.23} & $\mathbf{2.96}\times$ & \textbf{3.56} & $\mathbf{2.59}\times$ \\
\bottomrule
\end{tabular}
}
\end{table}

\clearpage
\subsection{Temperature Ablation}
\label{app:additional-controls}

\begin{table}[!htbp]
\centering
\small
\setlength{\tabcolsep}{4pt}
\renewcommand{\arraystretch}{1.12}
\caption{Domino with Qwen3-8B at different target temperatures at $B=16$, pooled over six datasets. All draft proposals use sparse64 candidate pools. Stochastic results combine three seeds. Greedy ($T=0$) results use seed 0 and are included only for Top1 chain and First pass; stochastic sampling and replay entries are omitted (--). Bold marks the best value in each stochastic-decoding column. The underlined verifier denotes the default configuration adopted by Domino.}
\label{tab:temperature-controls}
\begin{tabular}{lll*{3}{cc}}
\toprule
& & & \multicolumn{2}{c}{$T=1.0$} & \multicolumn{2}{c}{$T=0.6$} & \multicolumn{2}{c}{$T=0$} \\
\cmidrule(lr){4-5}\cmidrule(lr){6-7}\cmidrule(lr){8-9}
Structure & Sampling & Verifier & $\tau$ & TPS & $\tau$ & TPS & $\tau$ & TPS \\
\midrule
\multirow{3}{*}{Chain} & Top1 & \underline{Naive} & 4.54 & 126.0 & 5.27 & 145.9 & 5.82 & 161.2 \\
 & Random & RS & 5.07 & 139.6 & 5.54 & 152.3 & -- & -- \\
 & Random & BlockV & 5.29 & 145.4 & 5.69 & 156.2 & -- & -- \\
\midrule
Tree (First) & TopK & Naive & 5.68 & 155.1 & 6.35 & 173.4 & 6.85 & 187.2 \\
\midrule
\multirow{3}{*}{Tree (Replay)} & WoR & RRS & 5.72 & 155.3 & 6.32 & 171.2 & -- & -- \\
 & WoR & Traversal & 5.89 & 159.8 & 6.39 & 172.9 & -- & -- \\
 & Mixed & UniVer & \textbf{5.97} & \textbf{161.6} & \textbf{6.48} & \textbf{175.2} & -- & -- \\
\bottomrule
\end{tabular}
\end{table}

\begin{table}[!htbp]
\centering
\small
\setlength{\tabcolsep}{4pt}
\renewcommand{\arraystretch}{1.12}
\caption{DFlash2 with Muse Glimmer 30B at different target temperatures at $B=16$, pooled over six datasets. All draft proposals use DFlash2's native 16-token candidate pools. Stochastic results combine three seeds. Greedy ($T=0$) results use seed 0 and are included only for Top1 chain and First pass; stochastic sampling and replay entries are omitted (--). Bold marks the best value in each stochastic-decoding column. The underlined verifier denotes the default configuration adopted by DFlash2.}
\label{tab:temperature-controls-dflash2}
\begin{tabular}{lll*{3}{cc}}
\toprule
& & & \multicolumn{2}{c}{$T=1.0$} & \multicolumn{2}{c}{$T=0.6$} & \multicolumn{2}{c}{$T=0$} \\
\cmidrule(lr){4-5}\cmidrule(lr){6-7}\cmidrule(lr){8-9}
Structure & Sampling & Verifier & $\tau$ & TPS & $\tau$ & TPS & $\tau$ & TPS \\
\midrule
\multirow{3}{*}{Chain} & Top1 & Naive & 3.99 & 69.8 & 4.52 & 78.9 & 5.00 & 87.4 \\
 & Random & \underline{RS} & 4.81 & 82.3 & 4.99 & 85.7 & -- & -- \\
 & Random & BlockV & 4.99 & 85.3 & 5.11 & 87.5 & -- & -- \\
\midrule
Tree (First) & TopK & Naive & 4.87 & 83.0 & 5.45 & 92.4 & 5.85 & 99.6 \\
\midrule
\multirow{3}{*}{Tree (Replay)} & WoR & RRS & 5.21 & 88.6 & 5.54 & 94.0 & -- & -- \\
 & WoR & Traversal & \textbf{5.33} & \textbf{90.6} & 5.61 & 95.3 & -- & -- \\
 & Mixed & UniVer & 5.30 & 90.0 & \textbf{5.63} & \textbf{95.6} & -- & -- \\
\bottomrule
\end{tabular}
\end{table}

\paragraph{Temperature ablation.}
Domino and DFlash2--Muse exhibit similar temperature trends, although the magnitudes differ (Tables~\ref{tab:temperature-controls} and~\ref{tab:temperature-controls-dflash2}). At $T=1$, the probability distributions are more diffuse, giving sampling-based chains and trees a larger advantage over deterministic candidate selection. As temperature decreases and probability mass concentrates, this advantage diminishes. The greedy $T=0$ entries use seed 0; in this limit the distributions collapse to point masses and verification reduces to token matching.

\clearpage
\subsection{Latency Breakdown}
\label{app:latency-breakdown}

Table~\ref{tab:latency-breakdown} reports the latency breakdown using the first eight samples from each of GSM8K, HumanEval, and MT-Bench at $B=16$ and $T=1$. Domino runs on an RTX 3090; DFlash2--Muse runs on an RTX PRO 6000.

\begin{table}[!htbp]
\centering
\small
\setlength{\tabcolsep}{5pt}
\renewcommand{\arraystretch}{1.12}
\caption{Latency in ms per steady-state decoding round, pooled over the three datasets. Tree combines construction and replay.}
\label{tab:latency-breakdown}
\begin{tabular}{lrrrrrr}
\toprule
Method & Target & Draft & Others & Verifier & Tree & Total \\
\midrule
\multicolumn{7}{c}{\textbf{Domino with Qwen3-8B (RTX 3090)}} \\
\midrule
First pass & 27.951 & 5.875 & 0.843 & 0.131 & 0.993 & 35.794 \\
WoR/RRS & 27.848 & 5.868 & 0.817 & 0.128 & 1.274 & 35.935 \\
WoR/Traversal & 27.889 & 5.897 & 0.834 & 0.164 & 1.274 & 36.058 \\
Mixed/UniVer & 27.926 & 5.865 & 0.832 & 0.164 & 1.292 & 36.079 \\
\midrule
\multicolumn{7}{c}{\textbf{DFlash2 with Muse Glimmer 30B (RTX PRO 6000)}} \\
\midrule
First pass & 49.246 & 7.732 & 1.104 & 0.159 & 0.506 & 58.747 \\
WoR/RRS & 49.252 & 7.734 & 1.083 & 0.149 & 0.572 & 58.789 \\
WoR/Traversal & 49.237 & 7.743 & 1.084 & 0.170 & 0.572 & 58.807 \\
Mixed/UniVer & 49.204 & 7.741 & 1.083 & 0.173 & 0.590 & 58.791 \\
\bottomrule
\end{tabular}
\end{table}

Domino incurs a higher second-pass tree construction cost because its GRU-based correction must be rerun under the replayed prefixes, whereas DFlash2 only requires transition-score lookups and makes replay particularly inexpensive. Verification can add a small amount of latency as well. Together, the additional tree construction and verification costs are on the order of 1\% of Domino's total round latency, and are even smaller for DFlash2.

\subsection{Sparse Candidate Pools for Domino}
\label{app:domino-sparse-control}

As described in Section~\ref{sec:tree-construction}, Domino restricts its draft vocabulary to the 64 highest-scoring backbone tokens at each position (sparse64) and applies GRU corrections and normalization within this pool; verification and the output distribution remain full-vocabulary. We apply sparse64 to all Domino methods, including chains, so that chain and tree proposals share the same support. Table~\ref{tab:domino-sparse-control} shows that this does not weaken the chain baselines: sparse64 changes pooled acceptance length by less than 0.7\% for all three chains while increasing speedup ratio (SR) by 1.9\%--4.7\%.

\begin{table}[!htbp]
\centering
\small
\setlength{\tabcolsep}{5pt}
\renewcommand{\arraystretch}{1.12}
\caption{Domino chain baselines with sparse64 and full-vocabulary draft proposals at $T=1$ and $B=16$, pooled over six datasets and three seeds. SR is relative to the matched Qwen3-8B AR baseline, and changes are relative to full-vocabulary drafting.}
\label{tab:domino-sparse-control}
\begin{tabular}{lcccccc}
\toprule
& \multicolumn{3}{c}{Acceptance length} & \multicolumn{3}{c}{SR ($\times$)} \\
\cmidrule(lr){2-4}\cmidrule(lr){5-7}
Method & Sparse64 & Full vocab. & Change & Sparse64 & Full vocab. & Change \\
\midrule
Top1 & 4.539 & 4.570 & $-0.67\%$ & 3.524 & 3.458 & $+1.90\%$ \\
RS & 5.065 & 5.060 & $+0.11\%$ & 3.904 & 3.730 & $+4.66\%$ \\
BlockV & 5.291 & 5.308 & $-0.32\%$ & 4.066 & 3.902 & $+4.20\%$ \\
\bottomrule
\end{tabular}
\end{table}

\end{document}